\documentclass[11pt]{article}

\usepackage[a4paper,margin=1in]{geometry}
\usepackage{amsmath,amssymb,amsthm,mathtools}
\usepackage{mathrsfs}
\usepackage{bm}
\usepackage{enumitem}
\usepackage{graphicx}
\usepackage{booktabs}
\usepackage{microtype}
\usepackage{hyperref}
\usepackage{xcolor}
\hypersetup{hidelinks}

\newtheorem{theorem}{Theorem}[section]
\newtheorem{lemma}[theorem]{Lemma}
\newtheorem{proposition}[theorem]{Proposition}
\newtheorem{corollary}[theorem]{Corollary}
\theoremstyle{definition}
\newtheorem{definition}[theorem]{Definition}
\newtheorem{assumption}[theorem]{Assumption}
\theoremstyle{remark}
\newtheorem{remark}[theorem]{Remark}

\newcommand{\X}{\mathcal X}
\newcommand{\Acal}{\mathcal A}
\newcommand{\Hcal}{\mathscr H}
\newcommand{\Pcal}{\mathcal P}
\newcommand{\Bplain}{\mathcal B_H}
\newcommand{\Bsharp}{\mathcal B_H^{\sharp}}
\newcommand{\Dsharp}{\mathcal D_N^{\sharp}}
\newcommand{\Cclass}{\mathcal C_{N,V}}
\newcommand{\R}{\mathbb R}
\newcommand{\E}{\mathbb E}
\newcommand{\ip}[2]{\left\langle #1,#2\right\rangle}
\newcommand{\norm}[1]{\left\lVert #1\right\rVert}
\newcommand{\Rad}{\mathfrak R}
\newcommand{\Mrep}{\mathfrak M}
\usepackage{authblk}
\title{Resolution-Consistent Greedy Neural Approximation and Learning from Infinite-Dimensional Inputs}

\title{Resolution-Consistent Greedy Neural Approximation on Infinite-Dimensional Spaces}

\author[1]{Pablo M. Bern\'a}
\author[2]{Antonio Falcó}
\author[3]{Diego Mondéjar}

\affil[1]{
	Universidad Tecnológica Atlántico Mediterráneo -- UTAMED\\
	Facultad de Empresa Digital, Tecnología y Derecho\\
	Málaga, España\\
	\texttt{pablomanuel.berna@utamed.es}
}

\affil[2]{
Departamento de Matemáticas, Física y Ciencias Tecnológicas\\
Universidad Cardenal Herrera-CEU, CEU Universities\\ 
San Bartolomé 55, Alfara del Patriarca (Valencia), 46115, Spain\\
	\texttt{afalco@uchceu.es}
}

\affil[3]{
	Departamento de Matemáticas\\
	CUNEF Universidad\\ 
	28040, España\\
	\texttt{diego.mondejar@cunef.edu}
}

\date{}
\begin{document}
\maketitle

\begin{abstract}
We develop constructive approximation and learning guarantees for shallow neural
models with infinite-dimensional inputs observed through finitely many coordinates.
The analysis is based on a parameter-normalized neural dictionary and its associated
weighted variation class. Within this class, the approximation error separates into a
distribution-dependent coordinate-truncation term and a greedy finite-width term. For
empirical regression, a fully-corrective greedy procedure yields population guarantees
whose statistical complexity is uniform in the retained input resolution. The same
framework extends to Hilbert-valued responses without an explicit dependence on the
output dimension. The dimension-free statements are statistical, not computational:
selecting a new neuron still requires solving a nonconvex parameter-search problem.
The quasi-Polish construction underlying recent infinite-dimensional universal
approximation results provides a motivating example, and synthetic experiments
illustrate the predicted resolution, width, and sample-size regimes.
\end{abstract}

\section{Introduction}

Many learning problems are naturally formulated on inputs that are not
finite-dimensional vectors. Functional observations, trajectories, fields,
probability measures, and solutions of partial differential equations are
more naturally regarded as elements of infinite-dimensional spaces. Neural
approximation in this setting raises a difficulty that is largely absent in
standard finite-dimensional learning: before a model can be trained, the
input itself must usually be represented at a finite resolution.

This leads to three distinct sources of error. First, only finitely many
coordinates or measurements of an infinite-dimensional input can be retained.
Second, the approximating neural network has finite width. Third, the network
must be learned from finitely many observations. A natural quantitative theory
should therefore explain how approximation and learning depend simultaneously
on the input resolution, the number of selected neurons, and the sample size.

Recent work of Galimberti \cite{galimberti2026} establishes global $L^p$
universal approximation results for suitable neural architectures on
infinite-dimensional and quasi-Polish spaces. These results provide an
important qualitative foundation: after representing the input through a
countable family of scalar observables, finite neural models can approximate
broad classes of target maps. Universal approximation, however, is a density
statement. It does not by itself provide a constructive rule for choosing the
neurons, a convergence rate in the network width, or a statistical guarantee
when the network is trained from finite data.

The purpose of this paper is to develop such a quantitative theory for a
regularity class adapted to greedy neural approximation. We assume that the
input admits a countable coordinate representation whose magnitudes are
controlled by a square-summable envelope. A finite-resolution learner retains
only the first finitely many coordinates. The resulting loss of information
is measured by the mean-square size of the discarded tail under the data
distribution. This formulation is deliberately more general than the
quasi-Polish setting: the latter becomes an important application rather than
an assumption of the main theory.

To construct the approximating network, we use greedy selection. Starting
from the current residual, the algorithm repeatedly searches for a neural
unit that is strongly correlated with what remains unexplained. At the
population level this leads to a direct decomposition of the approximation
error into a resolution term and a finite-width term. At the empirical level
we use a fully-corrective version of the greedy procedure, closely related to
conditional-gradient methods, in order to keep the complexity of the
successive approximants under control.

A central role is played by a parameter-weighted variation class. Roughly
speaking, neural units with large internal parameters are assigned a larger
cost, and the target is assumed to admit a representation with finite total
weighted cost. This normalization is what allows the same regularity quantity
to control both approximation under coordinate truncation and statistical
complexity. It is important, however, that this is a genuine regularity
assumption. Although normalization does not change the linear span generated
by the neural units, it does change bounded variation classes. Our results
therefore provide rates for an explicit weighted neural class; they should not
be interpreted as quantitative rates for every target covered by the
underlying universal approximation theorem. We make this distinction precise
later and show, through a one-dimensional threshold example, that the weighted
and unweighted classes can behave very differently.

The main statistical phenomenon is that finer input resolution does not
necessarily lead to a larger estimation penalty. A naive finite-dimensional
analysis might suggest a complexity term increasing with the number of
retained coordinates. Instead, the coordinate representations considered here
remain uniformly bounded in a common Hilbert space. Exploiting this geometry,
together with the parameter normalization, gives a Rademacher complexity
estimate that is uniform in the input resolution. Consequently, the final
population bound separates into three contributions:
\[
\text{resolution error}
\;+\;
\text{finite-width error}
\;+\;
\text{statistical error},
\]
where the statistical contribution has no explicit dependence on the number
of retained input coordinates. In the quasi-Polish example motivating the
paper, the first two contributions exhibit the familiar inverse-resolution
and inverse-width behavior, while the statistical term has the standard
square-root dependence on sample size, up to logarithmic factors.

This resolution independence is statistical rather than computational.
Selecting the next neural unit still requires solving a nonconvex optimization
problem over its parameters, and the cost of that search may increase with
the retained resolution. Similar computational difficulties occur in convex
infinite-width neural formulations \cite{bach2017}. The experiments in this
paper therefore use a finite candidate dictionary. Their purpose is to
illustrate separately the resolution, width, and sampling effects predicted
by the theory, rather than to claim an efficient solution of the continuous
neuron-selection problem.

Our analysis builds on several established ideas. Greedy approximation and
its connection with statistical learning have a long history
\cite{barron2008}; variation-space formulations and conditional-gradient
methods for shallow neural networks were developed, among others, by
Bach \cite{bach2017}; and recent work gives refined analyses of orthogonal
greedy algorithms and neural variation spaces
\cite{siegelxuoga,siegelxuvariation}. The contribution here is not a new
greedy principle or a new empirical-process inequality. Rather, it is an
end-to-end analysis showing how these tools interact with finite-resolution
observation of an infinite-dimensional input.

More specifically, the paper makes the following contributions:
\begin{enumerate}[label=(\roman*)]
	
	\item We formulate finite-resolution neural approximation under minimal
	measurability assumptions and quantify the information lost by truncating an
	infinite-dimensional coordinate representation.
	
	\item We introduce a parameter-weighted neural variation class and prove
	constructive greedy approximation bounds that separate the effects of input
	resolution and network width.
	
	\item We prove a Rademacher complexity bound for the normalized neural
	dictionary that is uniform in the retained input resolution, and combine it
	with a fully-corrective greedy procedure to obtain an end-to-end population
	risk guarantee.
	
	\item We quantify the additional regularity imposed by the weighted class
	through a Lipschitz lower bound and a sharp one-dimensional threshold example.
	
	\item We extend the statistical argument to Hilbert-valued responses without
	introducing an explicit dependence on either the retained input dimension or
	the output dimension in the statistical term.
	
	\item We provide reproducible synthetic experiments that isolate the
	resolution, width, and sample-size regimes and illustrate the distinction
	between weighted and unweighted variation constraints.
	
\end{enumerate}

The remainder of the paper is organized as follows.
Section~\ref{sec:setting} introduces the coordinate representation and the
notion of resolution error. The following sections define the weighted neural
variation class and establish the deterministic greedy approximation results.
We then study empirical complexity and the fully-corrective learning
procedure, followed by the numerical experiments. Finally,
Section~\ref{sec:hilbert} treats Hilbert-valued responses before the discussion
and conclusion.

\section{Minimal measurable setting}\label{sec:setting}

\begin{assumption}[Measurable coordinate embedding]\label{ass:embedding}
Let $(\X,\Acal,\mu)$ be a probability space. Let
$h_j:\X\to\R$, $j\ge1$, be measurable and assume that there exists
$\alpha=(\alpha_j)_{j\ge1}\in\ell^2$ such that
\[
  |h_j(x)|\le \alpha_j
  \qquad\text{for all }j\ge1\text{ and }x\in\X.
\]
Define
\[
H(x):=(h_1(x),h_2(x),\ldots)\in\ell^2,
\qquad
H_N(x):=P_NH(x),
\]
where $P_N$ is the orthogonal projection onto the first $N$ canonical coordinates.
\end{assumption}

Two coordinate-tail quantities will be useful. The distribution-dependent tail is
\begin{equation}
\eta_N(\mu)
:=
\left(
\int_\X \norm{(I-P_N)H(x)}_{\ell^2}^2\,d\mu(x)
\right)^{1/2},
\label{eq:eta-mu}
\end{equation}
while the uniform tail is
\begin{equation}
\eta_N^{\infty}
:=
\sup_{x\in\X}\norm{(I-P_N)H(x)}_{\ell^2}.
\label{eq:eta-inf}
\end{equation}
By Assumption~\ref{ass:embedding},
\begin{equation}
\eta_N(\mu)
\le \eta_N^{\infty}
\le
\left(\sum_{j>N}\alpha_j^2\right)^{1/2}.
\label{eq:eta-envelope}
\end{equation}

We incorporate the bias as an additional Hilbert coordinate. Let
\[
\Pcal:=\ell^2\oplus\R,
\qquad
\theta=(w,b)\in\Pcal,
\qquad
\norm{\theta}_{\Pcal}:=\sqrt{\norm w_{\ell^2}^2+b^2},
\]
and define
\[
Z(x):=(H(x),1),
\qquad
Z_N(x):=(H_N(x),1).
\]
The envelope gives the uniform radius
\begin{equation}
K:=\sup_x\norm{Z(x)}_{\Pcal}
\le
\sqrt{1+\norm\alpha_{\ell^2}^2},
\label{eq:K}
\end{equation}
and $\norm{Z_N(x)}_{\Pcal}\le K$ for every $N$.

\subsection{The quasi-Polish construction as an example}

The topological assumptions used by Galimberti are not required for the estimates below.
They become relevant when one wants to connect the class to a global density theorem.
In the quasi-Polish construction of \cite{galimberti2026}, a continuous separating
sequence can be scaled so that
\[
0\le h_j(x)\le \frac1j.
\]
Consequently,
\begin{equation}
\eta_N(\mu)^2
\le (\eta_N^{\infty})^2
\le
\sum_{j>N}\frac1{j^2}
\le
\frac1N,
\qquad
K\le\sqrt{1+\frac{\pi^2}{6}}.
\label{eq:quasipolish-envelope}
\end{equation}
The separating property is used later only to explain density of the union of weighted
balls; it is not used in the quantitative proofs.

\section{Normalized neural atoms and weighted variation}\label{sec:variation}

We use the bounded $1$-Lipschitz activation
\begin{equation}
\rho(t):=\frac{t_+}{1+t_+},
\qquad t_+:=\max\{t,0\}.
\label{eq:activation}
\end{equation}
Thus $0\le\rho\le1$, $\rho(0)=0$, and
$|\rho(s)-\rho(t)|\le|s-t|$. Moreover,
\[
\lim_{\lambda\to\infty}\rho(\lambda t)
=
\begin{cases}
1,&t>0,\\
0,&t\le0,
\end{cases}
\]
so the associated rank-one infinite-dimensional activation has the separating
behavior used in \cite{galimberti2026}.

For $\theta=(w,b)$ define
\[
g_\theta(x)
:=
\rho(\ip{\theta}{Z(x)}_{\Pcal})
=
\rho(b+\ip w{H(x)}_{\ell^2}),
\]
and
\[
g_{\theta,N}(x)
:=
\rho(\ip{\theta}{Z_N(x)}_{\Pcal}).
\]

\begin{definition}[Parameter-normalized atoms]\label{def:atoms}
Let
\[
c(\theta):=1+\norm\theta_{\Pcal}.
\]
Define
\begin{equation}
\psi_\theta(x):=\frac{g_\theta(x)}{c(\theta)},
\qquad
\psi_{\theta,N}(x):=\frac{g_{\theta,N}(x)}{c(\theta)},
\label{eq:normalized-atoms}
\end{equation}
and the symmetric finite-resolution dictionary
\begin{equation}
\Dsharp:=\{\pm\psi_{\theta,N}:\theta\in\Pcal\}.
\label{eq:Dsharp}
\end{equation}
\end{definition}

Because $0\le\rho\le1$,
\begin{equation}
|\psi_{\theta,N}(x)|
\le
\frac1{1+\norm\theta_{\Pcal}}
\le1.
\label{eq:atom-envelope}
\end{equation}
The normalization preserves the linear span,
\begin{equation}
\operatorname{span}\{\psi_\theta:\theta\in\Pcal\}
=
\operatorname{span}\{g_\theta:\theta\in\Pcal\},
\label{eq:same-span}
\end{equation}
but this fact must not be confused with preservation of fixed variation balls.

\begin{definition}[Unweighted and weighted variation norms]\label{def:variation}
Let $\Mrep(f)$ be the collection of finite signed Borel measures $\nu$ on $\Pcal$ such
that
\begin{equation}
f(x)=\int_{\Pcal}g_\theta(x)\,d\nu(\theta)
\qquad\text{for }\mu\text{-a.e. }x.
\label{eq:rep}
\end{equation}
Define the unweighted variation seminorm
\begin{equation}
\norm f_{\Bplain}
:=
\inf_{\nu\in\Mrep(f)}|\nu|(\Pcal),
\label{eq:Bplain}
\end{equation}
and the weighted norm
\begin{equation}
\norm f_{\Bsharp}
:=
\inf_{\nu\in\Mrep(f)}
\int_{\Pcal}(1+\norm\theta_{\Pcal})\,d|\nu|(\theta).
\label{eq:Bsharp}
\end{equation}
The class $\Bsharp$ consists of functions with finite weighted norm.
\end{definition}

Since $c(\theta)\ge1$,
\begin{equation}
\norm f_{\Bplain}
\le
\norm f_{\Bsharp}.
\label{eq:ball-inclusion}
\end{equation}
Thus every radius-$V$ weighted ball is contained in the corresponding unweighted ball.
The normalization is therefore not free: it preserves the span but imposes additional
regularity on quantitative approximation and learning statements.

If $\nu\in\Mrep(f)$ has finite weighted cost, define the signed measure $\lambda$ by
$d\lambda=c\,d\nu$ in polar form. Then
\begin{equation}
f(x)=\int_{\Pcal}\psi_\theta(x)\,d\lambda(\theta),
\qquad
|\lambda|(\Pcal)=\int c(\theta)\,d|\nu|(\theta).
\label{eq:normalized-rep}
\end{equation}
This is the representation used throughout the proofs. Related parameter-measure and Barron-space viewpoints are developed in, for example, \cite{emawu2019}.

\begin{proposition}[Weighted variation controls the embedding Lipschitz seminorm]\label{prop:lipschitz-lb}
Define the pseudometric
\[
 d_H(x,y):=\norm{H(x)-H(y)}_{\ell^2}.
\]
Every $f\in\Bsharp$ admits a version satisfying
\begin{equation}
 |f(x)-f(y)|\le \norm f_{\Bsharp}\,d_H(x,y),
 \qquad x,y\in\X.
 \label{eq:lipschitz-lb}
\end{equation}
Consequently, whenever the $d_H$-Lipschitz seminorm of a version of $f$ is well defined,
\begin{equation}
 \operatorname{Lip}_{d_H}(f)\le \norm f_{\Bsharp}.
 \label{eq:lip-below-Bsharp}
\end{equation}
\end{proposition}

\begin{proof}
For every parameter $\theta=(w,b)$, the $1$-Lipschitz property of $\rho$ gives
\[
\begin{aligned}
 |\psi_\theta(x)-\psi_\theta(y)|
 &\le
 \frac{|\ip w{H(x)-H(y)}|}{1+\sqrt{\norm w^2+b^2}}\\
 &\le
 \frac{\norm w}{1+\sqrt{\norm w^2+b^2}}\,d_H(x,y)
 \le d_H(x,y).
\end{aligned}
\]
Let $f=\int\psi_\theta\,d\lambda(\theta)$ be any normalized representation with
$|\lambda|(\Pcal)<\infty$. Then
\[
 |f(x)-f(y)|
 \le \int |\psi_\theta(x)-\psi_\theta(y)|\,d|\lambda|(\theta)
 \le |\lambda|(\Pcal)d_H(x,y).
\]
Taking the infimum over all normalized representations proves
\eqref{eq:lipschitz-lb}--\eqref{eq:lip-below-Bsharp}.
\end{proof}

The preceding proposition also quantifies the additional regularity imposed by the weighted variation norm. This can be seen explicitly for the one-dimensional threshold family used later in the numerical experiments. As the threshold becomes sharper, its weighted variation norm necessarily grows linearly with the sharpness parameter, whereas its unweighted variation norm remains uniformly bounded.

\begin{corollary}\label{cor:threshold-linear}
	
	Let $\X=[-1,1]$ and $H(x)=x$. For $\lambda>0$, define
	
	\[
	f_\lambda(x):=\rho(\lambda x).
	\]
	Then
	\begin{equation}	\label{eq:threshold-sandwich}
			\lambda
			\le
			\norm{f_\lambda}_{\Bsharp}
			\le
			1+\lambda.
	\end{equation}
	At the same time,
	\[	
	\norm{f_\lambda}_{\Bplain}\le 1.
	\]
	Consequently, the weighted and unweighted fixed-radius variation balls separate quantitatively as $\lambda$ increases.
\end{corollary}

\begin{proof}
	Since $H(x)=x$, the pseudometric induced by the embedding is
	\[	
	d_H(x,y)
	=
	\norm{H(x)-H(y)}_{\ell^2}
	=
	|x-y|.	
	\]
	Thus $\operatorname{Lip}_{d_H}(f_\lambda)$ coincides with the usual Lipschitz seminorm of $f_\lambda$ on $[-1,1]$. We first show that
	\[	
	\operatorname{Lip}_{d_H}(f_\lambda)=\lambda.
	\]
	Because $\rho$ is $1$-Lipschitz,
	\[
	\begin{aligned}
		|f_\lambda(x)-f_\lambda(y)|
		&=
		|\rho(\lambda x)-\rho(\lambda y)| \\
		&\le
		|\lambda x-\lambda y| \\
		&=
		\lambda |x-y|.
	\end{aligned}	
	\]
	Hence $
	\operatorname{Lip}_{d_H}(f_\lambda)\le\lambda.$ 	To obtain the reverse inequality, take $h>0$. Since
	\[
	\rho(t)=\frac{t}{1+t},
	\qquad t>0,
	\]
	and $\rho(0)=0$, we have
	\[	
	\begin{aligned}
		\frac{|f_\lambda(h)-f_\lambda(0)|}{|h|}
		&=
		\frac{\rho(\lambda h)}{h} =
		\frac{1}{h}
		\frac{\lambda h}{1+\lambda h} =
		\frac{\lambda}{1+\lambda h}.
	\end{aligned}
	\]
	Letting $h\downarrow0$ gives
	\[	
	\lim_{h\downarrow0}
	\frac{|f_\lambda(h)-f_\lambda(0)|}{|h|}
	=
	\lambda.
	\]
	Since the Lipschitz seminorm is the supremum of the corresponding difference quotients, $
	\operatorname{Lip}_{d_H}(f_\lambda)\ge\lambda.$	Therefore
	\begin{equation}		
		\operatorname{Lip}_{d_H}(f_\lambda)=\lambda.
		\label{eq:threshold-lipschitz}
	\end{equation}
	Proposition~\ref{prop:lipschitz-lb} now yields
	\[
	\lambda
	=
	\operatorname{Lip}_{d_H}(f_\lambda)
	\le
	\norm{f_\lambda}_{\Bsharp},
	\]
	which proves the lower bound in	\eqref{eq:threshold-sandwich}.
	
	For the upper bound, observe that $f_\lambda$ is itself represented by a single raw atom. Indeed, taking
	\[	
	\theta_\lambda=(w,b)=(\lambda,0)
	\]
	gives
	\[
	g_{\theta_\lambda}(x)
	=
	\rho\bigl(0+\lambda H(x)\bigr)
	=
	\rho(\lambda x)
	=
	f_\lambda(x).
	\]
	Hence the Dirac measure
	\[
	\nu=\delta_{\theta_\lambda}
	\]
	belongs to $\Mrep(f_\lambda)$. Its weighted cost is
	\[	
	\begin{aligned}
		\int_{\Pcal}
		\bigl(1+\norm{\theta}_{\Pcal}\bigr)
		\,d|\nu|(\theta)
		&=
		1+\norm{\theta_\lambda}_{\Pcal} \\
		&=
		1+\sqrt{\lambda^2+0^2} \\
		&=
		1+\lambda.
	\end{aligned}
	\]
	Since the weighted variation norm is the infimum over all admissible representations,
	\[	
	\norm{f_\lambda}_{\Bsharp}
	\le
	1+\lambda.
	\]
	Together with the lower bound, this proves
	\[
	\lambda
	\le
	\norm{f_\lambda}_{\Bsharp}
	\le
	1+\lambda.
	\]
	Finally, the same single-atom representation gives
	\[
	|\delta_{\theta_\lambda}|(\Pcal)=1.
	\]
	By the definition of the unweighted variation seminorm,
	\[
	\norm{f_\lambda}_{\Bplain}
	\le
	1.
	\]
	This completes the proof.
\end{proof}

\begin{remark}[Interpretation]\label{rem:threshold}
	
	Corollary~\ref{cor:threshold-linear} shows that the cost induced by parameter normalization is intrinsic to the weighted variation class, rather than an artifact of the canonical single-neuron representation. Indeed,
	
	\[
	\lambda
	\le
	\norm{f_\lambda}_{\Bsharp}
	\le
	1+\lambda,
	\]
	and hence
	\[
	\norm{f_\lambda}_{\Bsharp}
	=
	\lambda+O(1)
	\qquad\text{as }\lambda\to\infty,
	\]
	whereas $\norm{f_\lambda}_{\Bplain}\le1$ uniformly in $\lambda$. Thus parameter normalization can substantially alter the geometry of fixed-radius variation balls even though it preserves the linear span of the neural dictionary.
	
	For the one-dimensional threshold family considered in Section~\ref{sec:threshold-experiment}, this means that increasingly sharp decision boundaries eventually leave every fixed-radius weighted variation ball, while remaining uniformly bounded in the corresponding unweighted variation seminorm. This conclusion is specific to the present family and does not imply an analogous lower bound for arbitrary separating families.
	\end{remark}

\begin{remark}[Density under Galimberti's hypotheses]\label{rem:density}
Under the quasi-Polish and separating assumptions of \cite{galimberti2026}, the span of
the corresponding ridge atoms is dense in the relevant $L^2(\mu)$ space. Every finite
linear combination has finite weighted cost. Therefore
\[
\bigcup_{V<\infty}\{f:\norm f_{\Bsharp}\le V\}
\]
is dense in $L^2(\mu)$. This density statement does not provide a uniform radius $V$
for arbitrary targets and should not be interpreted as a quantitative universal
approximation theorem.
\end{remark}

\section{Quantitative coordinate truncation}\label{sec:truncation}

Recall from Assumption~\ref{ass:embedding} that $P_N$ denotes the orthogonal
projection of $\ell^2$ onto its first $N$ canonical coordinates, so that
$H_N(x)=P_NH(x)$ and $(I-P_N)H(x)$ is the tail of $H(x)$ beyond coordinate $N$.

\begin{lemma}[Pointwise truncation of a normalized atom]\label{lem:atom-trunc}
For every $\theta=(w,b)\in\Pcal$, every $N$, and every $x\in\X$,
\begin{equation}
|\psi_\theta(x)-\psi_{\theta,N}(x)|
\le
\norm{(I-P_N)H(x)}_{\ell^2}.
\label{eq:atom-trunc-pointwise}
\end{equation}
Consequently,
\[
\norm{\psi_\theta-\psi_{\theta,N}}_{L^2(\mu)}\le\eta_N(\mu),
\qquad
\norm{\psi_\theta-\psi_{\theta,N}}_{L^\infty}\le\eta_N^\infty.
\]
\end{lemma}

\begin{proof}
By the Lipschitz property of $\rho$,
\[
\begin{aligned}
|\psi_\theta(x)-\psi_{\theta,N}(x)|
&\le
\frac{|\ip w{(I-P_N)H(x)}|}{1+\sqrt{\norm w^2+b^2}}\\
&\le
\norm{(I-P_N)H(x)}
\frac{\norm w}{1+\sqrt{\norm w^2+b^2}}\\
&\le
\norm{(I-P_N)H(x)}.
\end{aligned}
\]
The two norm estimates follow by integration and by taking the supremum.
\end{proof}

\begin{lemma}[Finite-resolution comparator]\label{lem:comparator}
Assume $f$ has a representation with weighted cost at most $V$, i.e.
\begin{equation}
f(x)=\int g_\theta(x)\,d\nu(\theta),
\qquad
\int c(\theta)\,d|\nu|(\theta)\le V.
\label{eq:costV}
\end{equation}
Let $\lambda$ be the corresponding measure in \eqref{eq:normalized-rep} and define
\begin{equation}
f_N(x):=\int_{\Pcal}\psi_{\theta,N}(x)\,d\lambda(\theta).
\label{eq:fN}
\end{equation}
Then, pointwise,
\begin{equation}
|f(x)-f_N(x)|
\le
V\norm{(I-P_N)H(x)},
\label{eq:fN-pointwise}
\end{equation}
and hence
\begin{equation}
\norm{f-f_N}_{L^2(\mu)}\le V\eta_N(\mu),
\qquad
\norm{f-f_N}_{L^\infty}\le V\eta_N^\infty.
\label{eq:fN-trunc}
\end{equation}
Moreover $f_N$ admits a normalized representation of total variation at most $V$.
\end{lemma}

\begin{proof}
Integrate the pointwise estimate from Lemma~\ref{lem:atom-trunc} against
$|\lambda|$ and use $|\lambda|(\Pcal)\le V$.
\end{proof}

\section{Population greedy approximation}\label{sec:population}

Let $\Hcal=L^2(\mu)$. For a symmetric dictionary $\mathcal D\subset\Hcal$ with
$\sup_{g\in\mathcal D}\norm g\le1$, let $\mathcal K_1(\mathcal D)$ denote the gauge of
its closed absolutely convex hull in $\Hcal$. Then
\begin{equation}
|\ip r h|
\le
\norm h_{\mathcal K_1(\mathcal D)}
\sup_{g\in\mathcal D}|\ip r g|.
\label{eq:atomic-dual}
\end{equation}

Starting from $f_{0,N}=0$, weak OGA selects $\phi_k\in\Dsharp$ satisfying
\begin{equation}
|\ip{f-f_{k-1,N}}{\phi_k}|
\ge
\gamma\sup_{\phi\in\Dsharp}|\ip{f-f_{k-1,N}}{\phi}|,
\qquad 0<\gamma<1,
\label{eq:woga}
\end{equation}
and lets $f_{k,N}$ be the orthogonal projection of $f$ onto the span of the selected
atoms.

\begin{lemma}[Recursion inversion]\label{lem:recursion}
Let $c>0$ and let $(a_k)_{k\ge0}$ be a sequence of nonnegative reals such that
\begin{equation}
a_k \le a_{k-1} - c\,a_{k-1}^2, \qquad k\ge1.
\label{eq:recursion-hyp}
\end{equation}
Then $a_1\le\dfrac1{4c}$, and
\begin{equation}
a_m \le \frac{1}{c(m+3)}
\qquad\text{for every integer } m\ge1.
\label{eq:recursion-concl}
\end{equation}
\end{lemma}

\begin{proof}
By \eqref{eq:recursion-hyp}, $a_k\le a_{k-1}$ for every $k$, so the sequence is
non-increasing and $a_k\in[0,a_0]$ for all $k$. The real map $x\mapsto x-cx^2$ attains
its maximum over $\R$ at $x=1/(2c)$ with value $1/(4c)$, so applying
\eqref{eq:recursion-hyp} with $k=1$ gives
\[
a_1\le a_0-ca_0^2\le\max_{x\in\R}(x-cx^2)=\frac1{4c}=\frac1{c(1+3)},
\]
which is exactly \eqref{eq:recursion-concl} at $m=1$.

For $k\ge2$ we have $a_{k-1}\le a_1\le\frac1{4c}$, hence $1-ca_{k-1}\ge\frac34>0$. If
$a_{k-1}=0$ then $a_k=0$ by \eqref{eq:recursion-hyp} and \eqref{eq:recursion-concl}
holds trivially at $m=k$ given that it holds at $m=k-1$. If $a_{k-1}>0$, then, using
$\frac1{1-x}\ge1+x$ for $x\in[0,1)$,
\[
\frac1{a_k}
\ge
\frac1{a_{k-1}-ca_{k-1}^2}
=
\frac1{a_{k-1}}\cdot\frac1{1-ca_{k-1}}
\ge
\frac1{a_{k-1}}\bigl(1+ca_{k-1}\bigr)
=
\frac1{a_{k-1}}+c.
\]
Telescoping this from $k=2$ to $k=m$ (for $m\ge2$) and using $a_1\le\frac1{4c}$,
\[
\frac1{a_m}
\ge
\frac1{a_1}+c(m-1)
\ge
4c+c(m-1)
=
c(m+3),
\]
i.e.\ $a_m\le\dfrac1{c(m+3)}$. Together with the case $m=1$ already established, this
proves \eqref{eq:recursion-concl} for every $m\ge1$.
\end{proof}

\begin{lemma}[Robust weak OGA]\label{lem:robust-woga}
For every $h\in\mathcal K_1(\mathcal D)$,
\begin{equation}
\norm{f-f_m}_{\Hcal}^2
\le
\norm{f-h}_{\Hcal}^2
+
\frac{4}{\gamma^2m}
\norm h_{\mathcal K_1(\mathcal D)}^2.
\label{eq:robust-woga}
\end{equation}
\end{lemma}

\begin{proof}
Set $r_k=f-f_k$, $d_k=\norm{r_k}^2$, $e=\norm{f-h}$, and
$M=\norm h_{\mathcal K_1(\mathcal D)}$. Orthogonality gives
$\ip{r_{k-1}}f=d_{k-1}$. If $d_{k-1}>e^2$, then
\[
|\ip{r_{k-1}}h|
\ge
d_{k-1}-e\sqrt{d_{k-1}}
\ge
\frac12(d_{k-1}-e^2).
\]
Using \eqref{eq:atomic-dual} and the weak selection rule,
\[
|\ip{r_{k-1}}{\phi_k}|
\ge
\frac{\gamma}{2M}(d_{k-1}-e^2).
\]
After removing the component of $\phi_k$ in the previous greedy span, the new
orthogonal direction has norm at most one and the same residual correlation. Hence
\[
d_k
\le
d_{k-1}-\frac{\gamma^2}{4M^2}(d_{k-1}-e^2)^2.
\]
Since orthogonal projection onto an enlarged subspace cannot increase the residual
norm, $d_k\le d_{k-1}$ for every $k$, so $\Delta_k:=(d_k-e^2)_+$ is non-increasing;
combined with the last display, this gives
$\Delta_k\le\Delta_{k-1}-\frac{\gamma^2}{4M^2}\Delta_{k-1}^2$ whenever $\Delta_{k-1}>0$,
and the same inequality holds trivially when $\Delta_{k-1}=0$. Applying
Lemma~\ref{lem:recursion} with $a_k=\Delta_k$ and $c=\gamma^2/(4M^2)$ yields
\begin{equation}
\Delta_m
\le
\frac{4M^2}{\gamma^2(m+3)}
\le
\frac{4M^2}{\gamma^2m},
\qquad m\ge1.
\label{eq:woga-sharp}
\end{equation}
\end{proof}

\begin{remark}
Estimate \eqref{eq:woga-sharp} is in fact slightly sharper than the bound
\eqref{eq:robust-woga} used below, and carries the same $m+3$ shift as the
Frank--Wolfe rate of Proposition~\ref{prop:fw}. We keep the weaker
$4M^2/(\gamma^2m)$ form in \eqref{eq:robust-woga} for notational symmetry with the
rest of the paper.
\end{remark}

\begin{theorem}[Deterministic width--resolution tradeoff]\label{thm:deterministic}
Under \eqref{eq:costV}, weak OGA over $\Dsharp$ satisfies
\begin{equation}
\boxed{
\norm{f-f_{m,N}}_{L^2(\mu)}^2
\le
V^2\eta_N(\mu)^2
+
\frac{4V^2}{\gamma^2m}.
}
\label{eq:det-main}
\end{equation}
In particular,
\begin{equation}
\norm{f-f_{m,N}}_{L^2(\mu)}
\le
V\left(\eta_N(\mu)+\frac{2}{\gamma\sqrt m}\right).
\label{eq:det-root}
\end{equation}
\end{theorem}

\begin{proof}
By Lemma~\ref{lem:comparator}, $f_N$ has atomic gauge at most $V$ and
$\norm{f-f_N}_2\le V\eta_N(\mu)$. Apply Lemma~\ref{lem:robust-woga} with $h=f_N$.
\end{proof}

\begin{corollary}[Quasi-Polish envelope]\label{cor:det-quasi}
If $|h_j|\le1/j$, then
\[
\norm{f-f_{m,N}}_{L^2(\mu)}
\le
V\left(\frac1{\sqrt N}+\frac{2}{\gamma\sqrt m}\right).
\]
\end{corollary}

\section{Rademacher complexity without an explicit resolution factor}\label{sec:rademacher}

For a fixed, deterministic sample $S=(x_1,\ldots,x_n)\in\X^n$ (lower-case points, not
random variables) and a class $\mathcal F$ of real-valued functions on $\X$ (a
\emph{real-valued class}), let
$\varepsilon=(\varepsilon_1,\ldots,\varepsilon_n)$ be i.i.d.\ Rademacher random variables
(uniform on $\{-1,+1\}$), independent of $S$, and define the empirical Rademacher
complexity
\begin{equation}
\widehat\Rad_S(\mathcal F)
:=
\E_\varepsilon\left[
\sup_{f\in\mathcal F}
\frac1n\sum_{i=1}^n\varepsilon_i f(x_i)
\right],
\label{eq:rad-def}
\end{equation}
where $\E_\varepsilon$ denotes expectation over $\varepsilon$ with $S$ held fixed. Thus
$\widehat\Rad_S(\mathcal F)$ is itself a deterministic number for each fixed $S$; it
becomes a random variable only through $S$ when the sample is later drawn at random,
as in Section~\ref{sec:empirical}, where $S=((X_1,Y_1),\ldots,(X_n,Y_n))$ is i.i.d.

Every class $\mathcal F$ appearing below (the dyadic shell classes $\mathcal F_{N,j}$,
the linear classes $\mathcal L_j$, and $\Dsharp$ itself) is indexed by a continuum of
parameters $\theta\in\Pcal$ rather than by a countable set, so $\sup_{f\in\mathcal F}$
and, once $S$ is random, quantities such as $\sup_{f\in\mathcal F}n^{-1}\sum_i
\varepsilon_if(X_i)$ require a measurability justification. Two distinct measurability
questions are involved, and it is worth separating them explicitly. The first, for
$S$ fixed, is whether $\E_\varepsilon\bigl[\sup_{f\in\mathcal F}n^{-1}\sum_i\varepsilon_i
f(x_i)\bigr]$ in \eqref{eq:rad-def} is well defined; this is immediate and requires no
argument at all, since $\varepsilon=(\varepsilon_1,\ldots,\varepsilon_n)$ ranges over
the \emph{finite} set $\{-1,+1\}^n$, so $\E_\varepsilon$ is literally a finite average
over $2^n$ terms, each of which is a supremum of a fixed real number
$n^{-1}\sum_i\varepsilon_if(x_i)$ over $f\in\mathcal F$ -- a real number (possibly
$+\infty$, though not here since $|f|\le1$ throughout), not a random variable, so no
measurability question about $\varepsilon$ arises.

The second, genuine question is whether, once the sample itself is random (as in
Section~\ref{sec:empirical}, with $S=((X_1,Y_1),\ldots,(X_n,Y_n))$ i.i.d.), maps such
as $S\mapsto\sup_{f\in\mathcal F}n^{-1}\sum_i\varepsilon_if(X_i)$ (for fixed
$\varepsilon$) or $S\mapsto\sup_{u\in\Cclass}|R_n(u)-\E R_n(u)|$ (in the
bounded-difference argument underlying Theorem~\ref{thm:stat-main}) are measurable
functions of $S$, as is implicitly required to apply McDiarmid's inequality and the
symmetrization step. This is resolved by the standard separability argument for
processes indexed by a continuous parameter (see, e.g.,
\cite{vandervaartwellner1996}, Section 2.1): for every fixed $x\in\X$ the map
$\theta\mapsto\psi_\theta(x)$ is continuous on $\Pcal$ (it is a continuous function of
$\ip w{H(x)}$ and $\norm\theta_{\Pcal}$, both continuous in $\theta$), and likewise for
$\theta\mapsto\psi_{\theta,N}(x)$ and for the map $(a_1,\ldots,a_k,\theta_1,\ldots,
\theta_k)\mapsto\sum_ja_j\psi_{\theta_j,N}(x)$ defining a general element of $\Cclass$.
Since $\Pcal=\ell^2\oplus\R$ is separable, any subset $\Theta\subseteq\Pcal$ has a
countable dense subset $\Theta_0$, and for every fixed sample the map
$\theta\mapsto(\varepsilon_i f_\theta(x_i))_{i\le n}$ is continuous, so
$\sup_{\theta\in\Theta}n^{-1}\sum_i\varepsilon_if_\theta(x_i)=\sup_{\theta\in\Theta_0}
n^{-1}\sum_i\varepsilon_if_\theta(x_i)$ for every fixed sample: the supremum over the
continuum coincides with a countable supremum. As a function of the (random) sample,
a countable supremum of measurable functions is measurable, so
$S\mapsto\sup_{\theta\in\Theta}n^{-1}\sum_i\varepsilon_if_\theta(X_i)$ is measurable,
and likewise for $\sup_{u\in\Cclass}|R_n(u)-\E R_n(u)|$ after replacing the continuum
of $(a,\theta)$-tuples defining $\Cclass$ by a countable dense subset (which exists
since $\Cclass$, as the image of a separable metric space under a continuous map, is
itself separable in $L^2(\mu)$ and in the sup norm on the sample). None of the
empirical-process bounds below are affected by working with the countable dense
subset in place of the continuum, so we do not restate this observation at each
subsequent supremum.

\begin{lemma}[Finite union of bounded classes]\label{lem:finite-union}
Let $\mathcal F_1,\ldots,\mathcal F_J$ be real-valued classes on $\X$ (in the sense of
\eqref{eq:rad-def}) satisfying $|f(x_i)|\le1$ on the sample for every $f$ in every
class. Then
\begin{equation}
\widehat\Rad_S\!\left(\bigcup_{j=1}^J\mathcal F_j\right)
\le
\max_{1\le j\le J}\widehat\Rad_S(\mathcal F_j)
+
\sqrt{\frac{2\log J}{n}}.
\label{eq:finite-union}
\end{equation}
\end{lemma}

\begin{proof}
For fixed Rademacher $\varepsilon$ signs let
\[
W_j(\varepsilon)
:=
\sup_{f\in\mathcal F_j}\frac1n\sum_{i=1}^n\varepsilon_i f(x_i).
\]
Changing one sign changes $W_j$ by at most $2/n$. The bounded-differences moment bound
therefore makes $W_j-\E_\varepsilon W_j$ sub-Gaussian with variance proxy $1/n$. The standard
maximal inequality for $J$ sub-Gaussian variables gives
\[
\E_\varepsilon\max_j W_j
\le
\max_j\E_\varepsilon W_j+\sqrt{\frac{2\log J}{n}},
\]
which is exactly \eqref{eq:finite-union}.
\end{proof}

\begin{lemma}[Dimension-uniform normalized-dictionary complexity]\label{lem:rad-main}
For every $n\ge2$, every $N$, and every sample $S$ (here $\log$ denotes the natural
logarithm),
\begin{equation}
\boxed{
\widehat\Rad_S(\Dsharp)
\le
\frac{4}{\sqrt n}
\left(K+1+\sqrt{\log(2+\log n)}\right).
}
\label{eq:rad-main}
\end{equation}
In particular, the bound has no explicit dependence on $N$, and the displayed universal constant is explicit. The factor is chosen conservatively to account for the explicit $\pm$ symmetrization of the normalized dictionary.
\end{lemma}

\begin{proof}
For $j\ge1$, let
\[
\Theta_j
:=
\{\theta\in\Pcal:2^{j-1}\le1+\norm\theta_{\Pcal}<2^j\},
\]
and
\[
\mathcal F_{N,j}:=\{\pm\psi_{\theta,N}:\theta\in\Theta_j\}.
\]
On $\Theta_j$,
\[
(1+\norm\theta_{\Pcal})^{-1}\le2^{1-j},
\qquad
\norm\theta_{\Pcal}<2^j.
\]

\paragraph{Shell bound, with an explicit contraction constant.}
The linear class
$\mathcal L_j:=\{x\mapsto\ip\theta{Z_N(x)}:\norm\theta_{\Pcal}<2^j\}$ is symmetric,
whereas $\rho\circ\mathcal L_j$ need not be symmetric because $\rho$ is not odd.
The shell dictionary is symmetric only because the sign $\pm$ is included explicitly
in its definition. We therefore first split the absolute supremum into its two signs
and then apply the one-sided contraction principle to each sign. In the form used
below (Ledoux and Talagrand \cite{ledouxtalagrand1991}, Cor.~3.17): for an $L$-Lipschitz $\varphi:\R\to\R$ with
$\varphi(0)=0$ and any class $\mathcal H$ of real functions on the sample,
\begin{equation}
\E_\varepsilon\sup_{h\in\mathcal H}\sum_{i=1}^n\varepsilon_i\varphi(h(x_i))
\le
L\,\E_\varepsilon\sup_{h\in\mathcal H}\sum_{i=1}^n\varepsilon_ih(x_i).
\label{eq:contraction}
\end{equation}
Since $\rho$ is $1$-Lipschitz and $\rho(0)=0$, contraction applies to each
one-sided term. By symmetry of the Rademacher signs the positive and negative terms
have the same expectation, and hence
\[
\begin{aligned}
\widehat\Rad_S(\mathcal F_{N,j})
&\le
\frac{2^{1-j}}{n}
\E_\varepsilon
\sup_{\theta\in\Theta_j}
\left|
\sum_{i=1}^n\varepsilon_i\rho(\ip\theta{Z_N(x_i)})
\right|\\
&\le
\frac{2^{2-j}}{n}
\E_\varepsilon
\sup_{\norm\theta_{\Pcal}<2^j}
\sum_{i=1}^n\varepsilon_i\ip\theta{Z_N(x_i)}\\
&=
\frac{2^{2-j}}{n}\cdot2^j\,
\E_\varepsilon
\norm{\sum_{i=1}^n\varepsilon_iZ_N(x_i)}_{\Pcal}
=
\frac{4}{n}
\E_\varepsilon
\norm{\sum_{i=1}^n\varepsilon_iZ_N(x_i)}_{\Pcal}
\le
\frac{4K}{\sqrt n},
\end{aligned}
\]
where the middle equality uses
$\sup_{\norm\theta_{\Pcal}<2^j}|\ip\theta v|=2^j\norm v_{\Pcal}$, and the last
inequality uses the Hilbert-space Khintchine bound
$\E_\varepsilon\norm{\sum_i\varepsilon_iZ_N(x_i)}_{\Pcal}\le\bigl(\sum_i\norm{Z_N(x_i)}_{\Pcal}^2\bigr)^{1/2}\le
K\sqrt n$. The bound $\widehat\Rad_S(\mathcal F_{N,j})\le4K/\sqrt n$ holds for every
$j\ge1$ with no unspecified constant.

\paragraph{Dyadic union.} Choose
\[
J:=\left\lceil\frac12\log_2n\right\rceil+1.
\]
By Lemma~\ref{lem:finite-union}, since every shell class is bounded by $1$ on the
sample (eq.~\eqref{eq:atom-envelope}),
\[
\widehat\Rad_S\Bigl(\bigcup_{j=1}^J\mathcal F_{N,j}\Bigr)
\le
\max_{1\le j\le J}\widehat\Rad_S(\mathcal F_{N,j})+\sqrt{\frac{2\log J}{n}}
\le
\frac{4K}{\sqrt n}+\sqrt{\frac{2\log J}n}.
\]
For $n\ge2$, $J\le\frac12\log_2n+2=\frac1{2\ln2}\ln n+2<\ln n+2$ (since
$\frac1{2\ln2}<1$), so $\log J\le\log(2+\log n)$ and hence
$\sqrt{2\log J/n}\le\sqrt2\,\sqrt{\log(2+\log n)}/\sqrt n$.

For $j>J$, every $\phi\in\mathcal F_{N,j}$ satisfies $|\phi|\le2^{-J}$ on the sample
(eq.~\eqref{eq:atom-envelope} together with $\Theta_j$'s definition), and
$2^{-J}\le2^{-\frac12\log_2n-1}=\tfrac12n^{-1/2}\le n^{-1/2}$; since a class uniformly
bounded by $\varepsilon$ on the sample has Rademacher complexity at most $\varepsilon$
regardless of how many functions it contains, the entire tail
$\bigcup_{j>J}\mathcal F_{N,j}$ contributes at most $n^{-1/2}$ to
$\widehat\Rad_S(\Dsharp)$. Combining the head and tail contributions,
\[
\begin{aligned}
\widehat\Rad_S(\Dsharp)
&\le
\frac{4K}{\sqrt n}+\sqrt2\,\frac{\sqrt{\log(2+\log n)}}{\sqrt n}+\frac1{\sqrt n}\\
&=
\frac1{\sqrt n}\Bigl(4K+1+\sqrt2\sqrt{\log(2+\log n)}\Bigr)\\
&\le
\frac4{\sqrt n}\Bigl(K+1+\sqrt{\log(2+\log n)}\Bigr).
\end{aligned}
\]
where the last step compares the three terms coefficient by coefficient
($4\ge4$, $4\ge1$, $4\ge\sqrt2$). This proves \eqref{eq:rad-main}; we may take $C_0=4$.
\end{proof}

\begin{remark}\label{rem:loglog}
The doubly logarithmic factor is only the cost of the elementary dyadic union argument;
each fixed parameter-norm shell already has complexity $O(K/\sqrt n)$ uniformly in
$N$. We do not emphasize removal of this factor as a primary contribution.
\end{remark}

\section{Fully-corrective empirical greedy regression}\label{sec:empirical}

Let $(X,Y)$ be distributed on $\X\times\R$ and let
$S=((X_1,Y_1),\ldots,(X_n,Y_n))$ be i.i.d. For a measurable prediction function
$u:\X\to\R$, define its empirical squared-error risk
\[
R_n(u):=\frac1n\sum_{i=1}^n(Y_i-u(X_i))^2.
\]
For $V>0$ define the prediction class directly through pointwise integral
representations: recalling the parameter space $\Pcal=\ell^2\oplus\R$ from
Section~\ref{sec:variation}, let $\lambda$ range over finite signed Borel measures on
$\Pcal$, with $|\lambda|$ its total-variation measure and $|\lambda|(\Pcal)$ the
corresponding total-variation norm (as in Definition~\ref{def:variation}), and set
\begin{equation}
\Cclass
:=
\left\{
 u:\ u(x)=\int_{\Pcal}\psi_{\theta,N}(x)\,d\lambda(\theta),\quad
 |\lambda|(\Pcal)\le V
\right\}.
\label{eq:Cclass}
\end{equation}
No $L^2$ closure is invoked in this definition. The class is convex, contains every
$\pm V\psi_{\theta,N}$, contains the comparator from Lemma~\ref{lem:comparator}, and
satisfies
\begin{equation}
|u(x)|\le V
\qquad\text{for every }u\in\Cclass.
\label{eq:prediction-envelope}
\end{equation}

\paragraph{Fully-corrective normalized greedy algorithm.}
Set $u_0=0$ and residuals $r_{k-1,i}=Y_i-u_{k-1}(X_i)$. The linear minimization oracle
over $\Cclass$ is equivalent, up to the irrelevant factor $V$, to selecting
\begin{equation}
\theta_k
\in
\operatorname*{arg\,max}_{\theta=(w,b)}
\frac{
\left|
\frac1n\sum_{i=1}^n r_{k-1,i}
\rho\!\left(b+\sum_{j=1}^Nw_jh_j(X_i)\right)
\right|
}{1+\sqrt{\norm w_{\ell^2}^2+b^2}}.
\label{eq:weighted-oracle}
\end{equation}
The domain $\theta=(w,b)\in\R^N\times\R$ in \eqref{eq:weighted-oracle} is not compact,
so existence of a maximizer is not automatic. It does hold, by a standard coercivity
argument: write $M:=\sup_\theta(\text{objective})\ge0$. If $M=0$ (the residual is
uncorrelated, on the sample, with every candidate atom), the supremum is trivially
attained at $\theta=0$. If $M>0$, the objective is continuous in $\theta$ (as $\rho$
is continuous) and tends to $0$ as $\norm\theta_{\Pcal}\to\infty$, because its
numerator is bounded by $n^{-1}\sum_i|r_{k-1,i}|$ while its denominator grows without
bound; hence there is $R>0$ with objective $<M/2$ whenever $\norm\theta_{\Pcal}>R$, so
every maximizing sequence eventually lies in the compact ball
$\{\norm\theta_{\Pcal}\le R\}$, on which the supremum is attained by continuity. This
settles existence but not tractability: as with Proposition~\ref{prop:fw} below, the
guarantees that follow assume access to an oracle that solves \eqref{eq:weighted-oracle}
exactly, and we make no claim that the underlying nonconvex selection problem can be
solved efficiently in practice.

After adding the selected signed atom, refit all active coefficients by
\begin{equation}
u_k
\in
\operatorname*{arg\,min}_{a_1,\ldots,a_k}
R_n\!\left(\sum_{j=1}^k a_j\psi_{\theta_j,N}\right)
\quad\text{subject to}\quad
\sum_{j=1}^k|a_j|\le V.
\label{eq:fully-corrective}
\end{equation}
Equivalently, if $\beta_j$ multiplies the unnormalized atom $g_{\theta_j,N}$, then
\[
\sum_{j=1}^k|\beta_j|(1+\norm{\theta_j}_{\Pcal})\le V.
\]
This is the same measure/conditional-gradient geometry that appears in convex neural
network formulations such as \cite{bach2017} and in projection-free convex optimization \cite{jaggi2013}; the full correction is used here to keep
all empirical iterates inside the statistically controlled ball. The negative result
of \cite{siegelxuoga} explains why an unconstrained orthogonal refit cannot generally
be assumed to have this property: the population algorithm of
Section~\ref{sec:population} can afford an unconstrained orthogonal refit because
coefficient growth causes no empirical-complexity issue there, whereas in the sampled
problem nearly dependent selected atoms can drive the variation-norm coefficients
arbitrarily large, so the empirical algorithm instead performs least squares under the
fixed weighted budget. The two correction steps are summarized here for reference.
\begin{center}
\begin{tabular}{p{0.25\textwidth}p{0.43\textwidth}p{0.22\textwidth}}
\toprule
Setting & Correction step & Purpose \\
\midrule
Population approximation (Section~\ref{sec:population}) & Orthogonal projection onto selected span & Hilbert-space OGA rate \\
Empirical learning (this section) & Least squares under weighted variation budget & Uniform statistical control \\
\bottomrule
\end{tabular}
\end{center}

\begin{proposition}[Empirical greedy optimization error]\label{prop:fw}
Assume the linear oracle is solved exactly and set
$R_n^*:=\inf_{u\in\Cclass}R_n(u)$. Then
\begin{equation}
R_n(u_m)-R_n^*
\le
\frac{16V^2}{m+3}.
\label{eq:optimization-gap}
\end{equation}
\end{proposition}

\begin{proof}
Use the empirical norm
$\norm u_n^2=n^{-1}\sum_i u(X_i)^2$. The feasible class has radius at most $V$ and
diameter at most $2V$. Let
$F(u)=R_n(u)$ and
$\Delta_k=F(u_k)-\inf_{u\in\Cclass}F(u)$. For an arbitrarily accurate minimizer $u^*$,
the exact linear oracle produces $q_k\in\Cclass$ with
\[
\ip{\nabla F(u_k)}{q_k}_n
\le
\ip{\nabla F(u_k)}{u^*}_n.
\]
By convexity,
$\ip{\nabla F(u_k)}{u_k-q_k}_n\ge\Delta_k$. For $\alpha\in[0,1]$, full correction is
at least as good as the feasible segment point $u_k+\alpha(q_k-u_k)$, and the quadratic
loss gives
\begin{equation}
\Delta_{k+1}
\le
(1-\alpha)\Delta_k+4V^2\alpha^2.
\label{eq:fw-recursion}
\end{equation}
Taking $\alpha=1$ at the first step yields $\Delta_1\le4V^2$. If
$\Delta_k\le16V^2/(k+3)$, choose $\alpha=2/(k+4)$ in
\eqref{eq:fw-recursion}; direct substitution gives
$\Delta_{k+1}\le16V^2/(k+4)$. Induction proves the claim.
\end{proof}

\begin{remark}[Inexact oracle and computational limitation]\label{rem:oracle}
If the linear oracle returns $q_k$ satisfying
\[
\ip{\nabla F(u_k)}{q_k}_n
\le
\inf_{q\in\Cclass}\ip{\nabla F(u_k)}q_n+\xi_k,
\]
then
\[
\Delta_{k+1}
\le
(1-\alpha)\Delta_k+\alpha\xi_k+4V^2\alpha^2.
\]
Thus an additive accuracy schedule $\xi_k=O(V^2/k)$ is sufficient to retain an
$O(V^2/m)$ optimization rate. The present theory does \emph{not} supply an algorithm
that guarantees such a schedule for the continuous oracle \eqref{eq:weighted-oracle}.
That oracle is a nonconvex optimization problem in $N+1$ variables, and its cost can
grow substantially with $N$; see also the computational discussion in
\cite{bach2017}. Therefore ``dimension-free'' in this paper means dimension-uniform
\emph{statistical complexity}, not dimension-free computation.
\end{remark}

\begin{remark}[Why ReLU is not covered by the same proof]\label{rem:relu}
Boundedness of $\rho$ is used twice: it makes the normalized atom envelope decay like
$(1+\norm\theta_{\Pcal})^{-1}$ on large parameter shells, and it yields a uniform prediction
envelope for the squared-loss contraction. For positively homogeneous ReLU,
normalization by the parameter norm behaves primarily as a reparameterization and the
dyadic tail does not shrink in the same way. ReLU may admit a different theory, but it
is not covered by Lemma~\ref{lem:rad-main} as stated.
\end{remark}

\section{End-to-end statistical guarantee}\label{sec:statistical}

Assume $|Y|\le B$ almost surely and let
$f_\star(x)=\E[Y\mid X=x]$. To avoid a spurious empirical-population discrepancy in
the $Y^2$ term, work with the centered squared loss
\begin{equation}
\ell_u(x,y):=u(x)^2-2yu(x).
\label{eq:centered-loss}
\end{equation}
Let
$L(u)=\E[\ell_u(X,Y)]$ and
$L_n(u)=n^{-1}\sum_i\ell_u(X_i,Y_i)$. Since
$R_n(u)=n^{-1}\sum_iY_i^2+L_n(u)$, minimizing $R_n$ is equivalent to minimizing $L_n$.
Moreover,
\begin{equation}
L(u)-L(f_\star)=\norm{u-f_\star}_{L^2(\mu)}^2.
\label{eq:regression-identity}
\end{equation}

\begin{theorem}[Dimension-uniform empirical greedy learning]\label{thm:stat-main}
Assume $|Y|\le B$ almost surely and suppose
\begin{equation}
f_\star(x)=\int g_\theta(x)\,d\nu(\theta),
\qquad
\int(1+\norm\theta_{\Pcal})\,d|\nu|(\theta)\le V.
\label{eq:target-cost}
\end{equation}
Let $\widehat f_{m,N}=u_m$ be the fully-corrective greedy estimator with an exact
linear oracle. Then there are universal constants $C_1,C_2>0$ such that, with
probability at least $1-\delta$,
\begin{equation}
\boxed{
\begin{aligned}
\norm{\widehat f_{m,N}-f_\star}_{L^2(\mu)}^2
\le{}&
V^2\eta_N(\mu)^2
+\frac{16V^2}{m+3}\\
&+
C_1(B+V)V
\frac{K+1+\sqrt{\log(2+\log n)}}{\sqrt n}\\
&+
C_2(B+V)^2
\sqrt{\frac{\log(2/\delta)}{n}}.
\end{aligned}
}
\label{eq:stat-main}
\end{equation}
The sampling terms contain no explicit dependence on the coordinate resolution $N$.
\end{theorem}

\begin{proof}
Let $f_N$ be the comparator from Lemma~\ref{lem:comparator}. Then
$f_N\in\Cclass$ and
$\norm{f_N-f_\star}_2\le V\eta_N(\mu)$. Proposition~\ref{prop:fw} gives
\begin{equation}
L_n(\widehat f_{m,N})
\le
L_n(f_N)+\frac{16V^2}{m+3}.
\label{eq:emp-comparison}
\end{equation}

For the statistical transfer, the integral definition of $\Cclass$ implies directly
that
\begin{equation}
\widehat\Rad_S(\Cclass)
\le
V\widehat\Rad_S(\Dsharp).
\label{eq:C-rad}
\end{equation}
Indeed, a linear functional is maximized over the total-variation ball by a signed
point mass. By Lemma~\ref{lem:rad-main}, the right-hand side is bounded by
\[
C_0V\frac{K+1+\sqrt{\log(2+\log n)}}{\sqrt n}.
\]

For $|y|\le B$ and $|u|\le V$, the map
$t\mapsto t^2-2yt$ vanishes at zero, is $2(B+V)$-Lipschitz on $[-V,V]$, and has
absolute value at most a universal multiple of $(B+V)^2$. Standard symmetrization,
contraction, and bounded-difference concentration
\cite{bartlettmendelson2002,ledouxtalagrand1991} therefore imply that, with
probability at least $1-\delta$,
\begin{equation}
\sup_{u\in\Cclass}|L(u)-L_n(u)|
\le
C_3(B+V)V
\frac{K+1+\sqrt{\log(2+\log n)}}{\sqrt n}
+
C_4(B+V)^2\sqrt{\frac{\log(2/\delta)}{n}}.
\label{eq:centered-uniform}
\end{equation}
Apply \eqref{eq:centered-uniform} to both $\widehat f_{m,N}$ and $f_N$ in
\eqref{eq:emp-comparison}, then use \eqref{eq:regression-identity} and
$\norm{f_N-f_\star}_2^2\le V^2\eta_N(\mu)^2$. Absorb numerical constants into
$C_1,C_2$.
\end{proof}

\begin{corollary}[Canonical quasi-Polish envelope]\label{cor:stat-quasi}
If $|h_j|\le1/j$, then with probability at least $1-\delta$,
\begin{align}
\norm{\widehat f_{m,N}-f_\star}_2^2
\le{}&
\frac{V^2}{N}
+\frac{16V^2}{m+3}
\nonumber\\
&+
C_1'(B+V)V
\frac{1+\sqrt{\log(2+\log n)}}{\sqrt n}
+
C_2(B+V)^2\sqrt{\frac{\log(2/\delta)}{n}}.
\label{eq:canonical-stat}
\end{align}
Thus the heuristic balance $N\asymp m\asymp\sqrt n$ matches the two deterministic
squared-error contributions to the global $n^{-1/2}$ sampling scale, up to logarithmic
factors.
\end{corollary}

\section{Synthetic experiments}\label{sec:experiment}

This section reports two complementary numerical illustrations, both built on finite
candidate dictionaries in place of the continuous nonconvex oracle.
Section~\ref{sec:three-regime} probes the three mechanisms separated by the end-to-end
guarantee of Theorem~\ref{thm:stat-main} (coordinate resolution, greedy width, and
sample size) one at a time. Section~\ref{sec:threshold-experiment} then isolates the
weighted-versus-unweighted contrast quantified by Corollary~\ref{cor:threshold-linear}.
Neither experiment is intended to validate theorem
constants or to demonstrate that the continuous oracle is computationally tractable.

\subsection{Three-regime synthetic experiment: resolution, width, and sample size}
\label{sec:three-regime}

The theorem separates three mechanisms, but the continuous oracle in
\eqref{eq:weighted-oracle} is computationally nontrivial. We therefore use a deliberately
simple finite-candidate experiment whose purpose is only to illustrate the three
regimes.

\paragraph{Data and target.}
We truncate a synthetic input at a maximal ambient dimension $D=48$ and draw
independent $\xi_j\sim\mathrm{Unif}[-1,1]$, setting
\[
h_j(x)=\frac{\xi_j}{j}.
\]
Thus the canonical envelope $|h_j|\le1/j$ holds exactly. The noiseless target is a
finite combination of $20$ normalized atoms. The absolute sum of its normalized outer
coefficients is $7.36$, below the algorithmic budget $V=8$. The candidate dictionary
contains the $20$ target atoms plus $108$ independently generated distractor atoms.
Hence the experiment is favorable to the finite oracle: at full resolution the target
is contained in the candidate variation ball.

At each greedy step we maximize residual correlation over these $128$ candidates and
then approximately solve the fully-corrective $\ell^1$-constrained least-squares problem
by projected accelerated gradient. A held-out set of $6000$ points is used for all
reported test errors. Curves show means and standard deviations over three repetitions.

The complete generation and optimization code could be found in \url{https://doi.org/10.5281/zenodo.22026280}.

\paragraph{Resolution sweep.}
To isolate truncation, we use $4096$ noiseless training observations and $m=20$ greedy
steps while varying $N$. Figure~\ref{fig:resolution} shows a monotone decrease in test
error as more coordinates are retained. The dashed $N^{-1}$ line is only the
worst-case squared-error reference suggested by the canonical envelope; the theorem
does not predict equality with that slope for this distribution and target. The final
sharp drop at $N=48$ reflects the finite ambient construction and the fact that all
true atoms are present in the candidate pool.

\begin{figure}[htbp]
\centering
\includegraphics[width=0.72\linewidth]{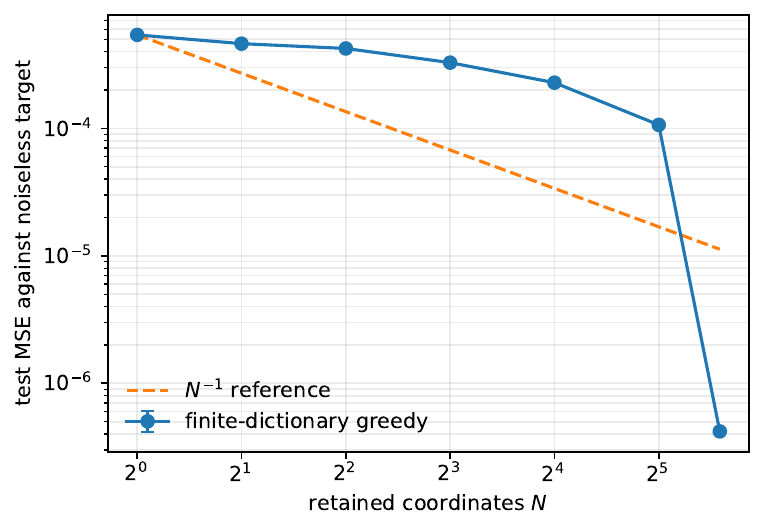}
\caption{Resolution sweep. Training labels are noiseless and $m=20$ is fixed. The
dashed line is an $N^{-1}$ reference, not a fitted law.}
\label{fig:resolution}
\end{figure}

\paragraph{Width sweep.}
With full resolution $N=48$ and $4096$ noiseless observations, increasing the number of
greedy atoms reduces the test error by more than three orders of magnitude, as shown in
Figure~\ref{fig:width}. Because the target itself is a $20$-atom combination included in
the pool, the error becomes nearly numerical at $m=20$. The dashed $m^{-1}$ line again
serves only as the theorem's generic squared-error reference.

\begin{figure}[htbp]
\centering
\includegraphics[width=0.72\linewidth]{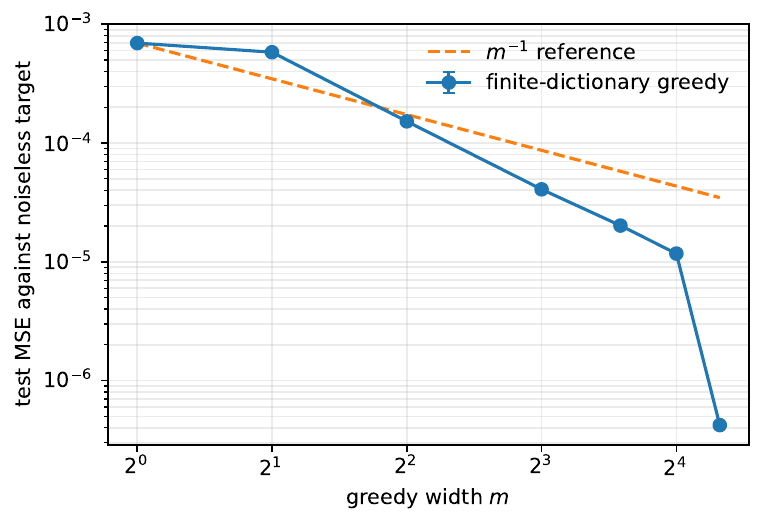}
\caption{Width sweep at full resolution. The finite candidate pool contains the target
atoms, so the last point is intentionally an easy endpoint.}
\label{fig:width}
\end{figure}

\paragraph{Sample-size sweep.}
Finally, fix $N=48$ and $m=20$, and add independent bounded noise with standard deviation
$0.08$ (uniform on $[-\sqrt3\,0.08,\sqrt3\,0.08]$). Figure~\ref{fig:sample} plots test
MSE against the noiseless regression function. Across the tested range the curve is
consistent with a roughly $n^{-1/2}$ global-complexity scale, although no asymptotic
claim is made from this small experiment.

\begin{figure}[htbp]
\centering
\includegraphics[width=0.72\linewidth]{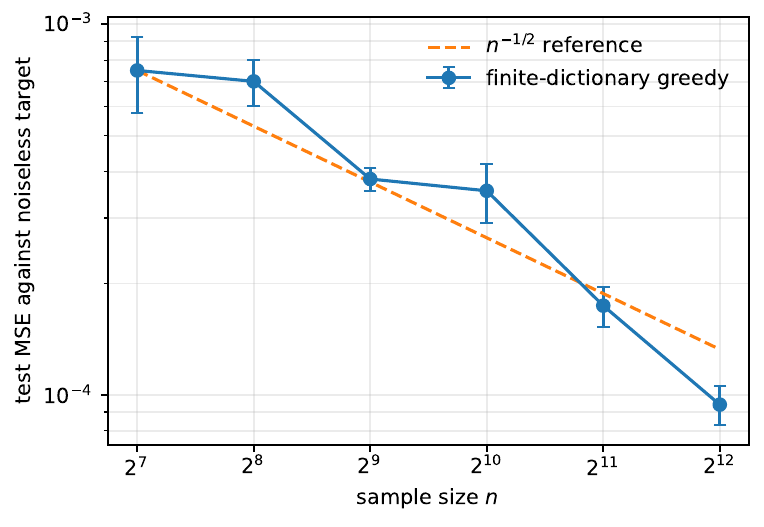}
\caption{Sample-size sweep with bounded observation noise. The dashed line is an
$n^{-1/2}$ reference.}
\label{fig:sample}
\end{figure}

The experiment should therefore be read as a sanity check for the decomposition rather
than evidence that the continuous oracle is tractable. A stronger computational paper
would need either a provable approximate neuron oracle or a realistic operator-learning
benchmark in which the cost of increasing $N$ is measured explicitly.

\subsection{Weighted versus unweighted variation: a numerical illustration}
\label{sec:threshold-experiment}

Corollary~\ref{cor:threshold-linear} proves in the present one-dimensional setting that
$\lambda\le\norm{f_\lambda}_{\Bsharp}\le1+\lambda$, whereas
$\norm{f_\lambda}_{\Bplain}\le1$. We complement this exact class separation with a
small numerical experiment that puts the two variation classes side by side under a
matched, finite representation budget.

\paragraph{Setup.} We work with a single coordinate $h_1(x)=x$, $x\in[-1,1]$, and
$c=0$, matching Remark~\ref{rem:threshold} exactly (so $p_+=\mu(h_1>0)=\tfrac12$ under
the uniform reference measure used for evaluation). For a grid of sharpness values
$\lambda\in\{1,2,4,\dots,256\}$ we fit the target $f_\lambda$ on $4000$ evenly spaced
test points using a finite candidate pool of $130$ atoms $g_{(w,b)}$, built from a
grid of $10$ slopes $w\in\{0.5,1,2,\dots,256\}$ and $13$ thresholds
$b\in[-0.6,0.6]$ -- deliberately richer than the single canonical atom, so that the
weighted fit is free to combine several affordable atoms if that would let it track a
sharp target more cheaply than the canonical representation. Each target is fit twice,
by the same fully-corrective projected-gradient scheme used in
Section~\ref{sec:three-regime} with budget $V=8$: once over the raw atoms $g_{(w,b)}$
with the classical unweighted budget $\sum_i|a_i|\le V$, and once over the normalized
atoms $\psi_{(w,b)}=g_{(w,b)}/(1+\sqrt{w^2+b^2})$ with the weighted budget
$\sum_i|a_i|\le V$ used throughout Sections~\ref{sec:empirical}--\ref{sec:statistical}.
The complete code is provided with the manuscript

\begin{figure}[htbp]
\centering
\includegraphics[width=0.95\linewidth]{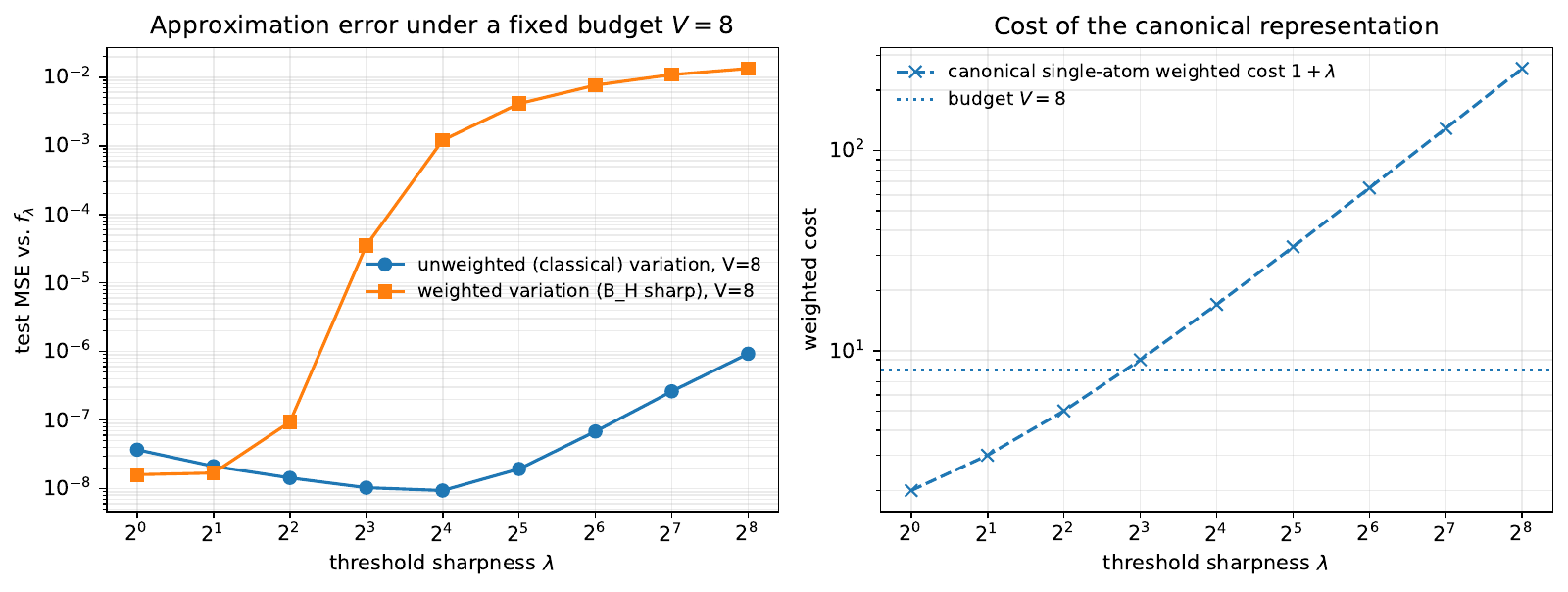}
\caption{Left: test MSE against $f_\lambda$ under a matched budget $V=8$, for the
unweighted (classical) and weighted ($\Bsharp$) variation classes. Right: the
canonical single-atom weighted cost $1+\lambda$, which is an upper bound on
$\norm{f_\lambda}_{\Bsharp}$ and differs by one from the new lower bound $\lambda$ in
Corollary~\ref{cor:threshold-linear}; the canonical cost crosses $V=8$ at $\lambda=7$.}
\label{fig:threshold}
\end{figure}

\paragraph{Result and relation to Corollary~\ref{cor:threshold-linear}.}
Figure~\ref{fig:threshold} shows a clean separation between the two fixed-radius
classes. Corollary~\ref{cor:threshold-linear} now gives a rigorous explanation in this
one-dimensional setting:
\[
 \lambda\le \norm{f_\lambda}_{\Bsharp}\le1+\lambda,
 \qquad
 \norm{f_\lambda}_{\Bplain}\le1.
\]
Thus a fixed weighted budget $V=8$ cannot contain $f_\lambda$ once $\lambda>8$, while
the corresponding unweighted ball already contains the exact target for every
$\lambda$ through its single raw atom. The numerical curves are consistent with this
sharp theoretical distinction: the unweighted fit remains near numerical accuracy,
whereas the weighted fit deteriorates rapidly as $\lambda$ passes the fixed budget.
The right panel displays the canonical upper bound $1+\lambda$; the new lower bound
$\lambda$ differs from it by only one unit, so the weighted norm is pinned down to a
unit-width interval for this family.

\paragraph{Computational caveat.}
The experiment still uses a finite candidate pool and therefore does not certify that
the projected-gradient solver finds the best continuum representation at a given
budget. This limitation concerns the numerical optimizer, not the class separation:
Corollary~\ref{cor:threshold-linear} proves independently of the candidate pool that
$f_\lambda$ lies outside the radius-$V$ weighted ball whenever $\lambda>V$.

\section{Hilbert-valued targets}\label{sec:hilbert}

Let $\mathcal Y$ be a separable real Hilbert space (so $\mathcal Y\cong\R^d$ or
$\mathcal Y\cong\ell^2$) with orthonormal basis $(s_k)_{k\ge1}$, and let
$F:\X\to\mathcal Y$. We extend the scalar atoms by an output direction,
\begin{equation}
\Psi_{\theta,v}(x):=\psi_\theta(x)v,
\qquad
\theta\in\Pcal,\quad v\in\mathcal Y,\ \norm v_{\mathcal Y}\le1,
\label{eq:vector-atom}
\end{equation}
so that $\norm{\Psi_{\theta,v}(x)}_{\mathcal Y}\le1$ pointwise, exactly as in
\eqref{eq:atom-envelope}.

\begin{lemma}[Linear oracle for Hilbert-valued targets]\label{lem:vector-oracle}
For every $R\in L^2(\mu;\mathcal Y)$ and every $\theta\in\Pcal$,
\begin{equation}
\sup_{\norm v_{\mathcal Y}\le1}
\bigl|\ip R{\psi_\theta v}_{L^2(\mu;\mathcal Y)}\bigr|
=
\norm{\int_\X\psi_\theta(x)R(x)\,\mu(dx)}_{\mathcal Y},
\label{eq:vector-oracle}
\end{equation}
and, when the vector on the right is nonzero, the supremum on the left is attained at
its normalization.
\end{lemma}

\begin{proof}
Write $\zeta(\theta):=\int_\X\psi_\theta(x)R(x)\,\mu(dx)\in\mathcal Y$ (well defined
since $\norm{\psi_\theta}_{L^\infty}\le1$ and $R\in L^2(\mu;\mathcal Y)\subset
L^1(\mu;\mathcal Y)$ on the probability space $(\X,\mu)$). By Fubini for the scalar
pairing,
\[
\ip R{\psi_\theta v}_{L^2(\mu;\mathcal Y)}
=
\int_\X\psi_\theta(x)\ip{R(x)}{v}_{\mathcal Y}\,\mu(dx)
=
\ip{\zeta(\theta)}{v}_{\mathcal Y}.
\]
Hence the left side of \eqref{eq:vector-oracle} equals
$\sup_{\norm v\le1}|\ip{\zeta(\theta)}v|=\norm{\zeta(\theta)}_{\mathcal Y}$ by the
Riesz representation of the norm on a Hilbert space, attained at
$v=\zeta(\theta)/\norm{\zeta(\theta)}$ whenever $\zeta(\theta)\ne0$.
\end{proof}

Lemma~\ref{lem:vector-oracle} shows that the output direction of the linear
minimization oracle \eqref{eq:weighted-oracle} can be eliminated analytically -- the
greedy step still reduces to a search over the scalar parameter $\theta$ alone, with
the optimal $v$ read off in closed form -- and that the coordinate-truncation argument
of Section~\ref{sec:truncation} is unchanged, since it acts only on the scalar input
atom $\psi_\theta$, not on $v$. This is a genuine computational simplification, but it
says nothing yet about statistical complexity, which is where the vector-valued case
departs from the scalar one.

\begin{definition}[Weighted vector variation class]\label{def:vector-variation}
For $F\in L^2(\mu;\mathcal Y)$, let $\Mrep_{\mathcal Y}(F)$ be the collection of pairs
$(\nu,w)$ where $\nu$ is a finite nonnegative Borel measure on $\Pcal$ and
$w:\Pcal\to\mathcal Y$ is $\nu$-measurable with $\norm{w(\theta)}_{\mathcal Y}\le1$
for $\nu$-a.e.\ $\theta$, such that
\[
F(x)=\int_\Pcal\psi_\theta(x)\,w(\theta)\,\nu(d\theta)
\qquad\text{for }\mu\text{-a.e.\ }x.
\]
Define
\begin{equation}
\norm F_{\Bsharp(\mathcal Y)}
:=
\inf_{(\nu,w)\in\Mrep_{\mathcal Y}(F)}\nu(\Pcal).
\label{eq:vector-variation}
\end{equation}
\end{definition}

This mirrors the $\lambda$-representation \eqref{eq:normalized-rep} of
Definition~\ref{def:variation}, not the raw $\nu$-representation \eqref{eq:rep}: since
$\psi_\theta$ already carries the weight $1/(1+\norm\theta_{\Pcal})$, the cost of a
representation is simply the total mass $\nu(\Pcal)$, with \emph{no} further
reweighting by $(1+\norm\theta_{\Pcal})$ -- reintroducing that factor here would
double-count it. For $\mathcal Y=\R$ (so $w:\Pcal\to[-1,1]$), we check
$\norm\cdot_{\Bsharp(\R)}=\norm\cdot_{\Bsharp}$ by the two inequalities that make up
equality of the two infima.

First, $\norm f_{\Bsharp}\le\norm f_{\Bsharp(\R)}$: given any $(\nu,w)\in\Mrep_{\R}(f)$,
set $d\lambda:=w\,d\nu$, a signed measure with $f=\int_\Pcal\psi_\theta\,d\lambda$ (a
valid $\lambda$-representation in the sense of \eqref{eq:normalized-rep}) and, since
$|w(\theta)|\le1$ pointwise and $\nu\ge0$, $|\lambda|=|w|\,\nu\le\nu$ as measures --
so $|\lambda|(\Pcal)\le\nu(\Pcal)$: no mass is discarded, it is absorbed into
$\lambda$ and reweighted by $|w|$ instead. Hence
$\norm f_{\Bsharp}\le|\lambda|(\Pcal)\le\nu(\Pcal)$ for every
$(\nu,w)\in\Mrep_{\R}(f)$; taking the infimum over $(\nu,w)$ gives $\norm
f_{\Bsharp}\le\norm f_{\Bsharp(\R)}$.

Second, $\norm f_{\Bsharp(\R)}\le\norm f_{\Bsharp}$: given any $\lambda$-representation
$f=\int_\Pcal\psi_\theta\,d\lambda$, set $\nu:=|\lambda|$ and
$w:=\operatorname{sgn}(\lambda)$ (defined $|\lambda|$-a.e.\ and extended arbitrarily,
say $w\equiv1$, on the $|\lambda|$-null remainder, which does not affect $\nu$ or the
integral). Then $(\nu,w)\in\Mrep_{\R}(f)$, since $\nu\ge0$ is finite, $|w|\le1$
$\nu$-a.e., and $\int_\Pcal\psi_\theta w\,d\nu=\int_\Pcal\psi_\theta\,d\lambda=f$, at
the same cost $\nu(\Pcal)=|\lambda|(\Pcal)$. Hence $\norm
f_{\Bsharp(\R)}\le\nu(\Pcal)=|\lambda|(\Pcal)$ for every $\lambda$-representation;
taking the infimum over $\lambda$ gives $\norm f_{\Bsharp(\R)}\le\norm f_{\Bsharp}$.

Together, $\norm\cdot_{\Bsharp(\mathcal Y)}$ reduces exactly to $\norm\cdot_{\Bsharp}$
when $\mathcal Y=\R$.

For the statistical result it is useful to define the vector-valued feasible class
\begin{equation}
\Cclass^{\mathcal Y}
:=
\left\{
 F(x)=\int_{\Pcal}\psi_{\theta,N}(x)v(\theta)\,\nu(d\theta):
 \nu\ge0,\ \nu(\Pcal)\le V,\ \norm{v(\theta)}_{\mathcal Y}\le1
\right\}.
\label{eq:vector-Cclass}
\end{equation}
Every $F\in\Cclass^{\mathcal Y}$ satisfies $\norm{F(x)}_{\mathcal Y}\le V$.

\begin{lemma}[Dimension-uniform vector loss complexity]\label{lem:vector-loss}
Let $S=((x_i,\mathbf y_i))_{i=1}^n$ satisfy $\norm{\mathbf y_i}_{\mathcal Y}\le B$.
For the centered vector squared loss
\[
 \ell_F(x,\mathbf y):=\norm{F(x)}_{\mathcal Y}^2-2\ip{\mathbf y}{F(x)}_{\mathcal Y},
\]
there is a universal constant $C>0$ such that
\begin{equation}
 \widehat\Rad_S\bigl(\ell\circ\Cclass^{\mathcal Y}\bigr)
 \le
 C\,(V^2+BV)\,
 \frac{K+1+\sqrt{\log(2+\log n)}}{\sqrt n}.
 \label{eq:vector-loss-rad}
\end{equation}
The bound is independent of both $N$ and $\dim(\mathcal Y)$.
\end{lemma}

\begin{proof}
Write $A_n:=K+1+\sqrt{\log(2+\log n)}$. For the linear part, by the integral
representation and duality in $\mathcal Y$,
\[
\begin{aligned}
&\E_\varepsilon\sup_{F\in\Cclass^{\mathcal Y}}
 \left|\frac1n\sum_i\varepsilon_i\ip{\mathbf y_i}{F(x_i)}\right|\\
&\qquad\le
V\,\E_\varepsilon\sup_{\theta,\norm v\le1}
 \left|\frac1n\sum_i\varepsilon_i\psi_{\theta,N}(x_i)\ip{\mathbf y_i}{v}\right|.
\end{aligned}
\]
The last display is the Rademacher complexity of the product of two classes bounded by
one after dividing the second factor by $B$. A standard bounded product-class contraction bound (obtainable, for example, from
Maurer's vector-contraction inequality applied to $(a,b)\mapsto ab$ on $[-1,1]^2$
\cite{maurer2016}) gives
\[
 \widehat\Rad_S(\mathcal F\mathcal G)
 \le C\bigl(\widehat\Rad_S(\mathcal F)+\widehat\Rad_S(\mathcal G)\bigr)
\]
for $[-1,1]$-valued classes. Here
$\mathcal F=\Dsharp$ and
$\mathcal G=\{i\mapsto\ip{\mathbf y_i}{v}/B:\norm v\le1\}$, while
\[
 \widehat\Rad_S(\mathcal G)
 =\frac1{nB}\E_\varepsilon\norm{\sum_i\varepsilon_i\mathbf y_i}_{\mathcal Y}
 \le\frac1{\sqrt n}.
\]
Lemma~\ref{lem:rad-main} therefore bounds the linear part by
$CBV A_n/\sqrt n$.

For the quadratic part, every $F\in\Cclass^{\mathcal Y}$ has a representation with
mass at most $V$, and
\[
 \norm{F(x_i)}^2
 =\iint \psi_{\theta,N}(x_i)\psi_{\theta',N}(x_i)
 \ip{v(\theta)}{v(\theta')}\,\nu(d\theta)\nu(d\theta').
\]
Hence
\[
\E_\varepsilon\sup_F\left|\frac1n\sum_i\varepsilon_i\norm{F(x_i)}^2\right|
\le
V^2\,
\E_\varepsilon\sup_{\theta,\theta'}
\left|\frac1n\sum_i\varepsilon_i
\psi_{\theta,N}(x_i)\psi_{\theta',N}(x_i)\right|.
\]
A second application of the bounded product-class inequality and
Lemma~\ref{lem:rad-main} gives $CV^2A_n/\sqrt n$. Combining the two parts proves
\eqref{eq:vector-loss-rad}. No coordinate basis of $\mathcal Y$ is used.
\end{proof}

\begin{theorem}[Dimension-uniform Hilbert-valued greedy learning]\label{thm:vector-stat}
Let $(X,\mathbf Y)$ be distributed on $\X\times\mathcal Y$ with
$\norm{\mathbf Y}_{\mathcal Y}\le B$ almost surely, and let
$F_\star(x)=\E[\mathbf Y\mid X=x]$. Suppose
$\norm{F_\star}_{\Bsharp(\mathcal Y)}\le V$.
Run the fully-corrective vector-valued greedy method over
$\Cclass^{\mathcal Y}$, using the exact linear oracle from
Lemma~\ref{lem:vector-oracle}, and denote its $m$th iterate by
$\widehat F_{m,N}$. Then universal constants $C_1,C_2>0$ exist such that, with
probability at least $1-\delta$,
\begin{equation}
\begin{aligned}
\norm{\widehat F_{m,N}-F_\star}_{L^2(\mu;\mathcal Y)}^2
\le{}&
V^2\eta_N(\mu)^2+\frac{16V^2}{m+3}\\
&+C_1(V^2+BV)
\frac{K+1+\sqrt{\log(2+\log n)}}{\sqrt n}
+C_2(B+V)^2\sqrt{\frac{\log(2/\delta)}n}.
\end{aligned}
\label{eq:vector-stat-uniform}
\end{equation}
In particular the statistical terms contain no explicit dependence on the retained
input resolution $N$ or on the Hilbert dimension of $\mathcal Y$.
\end{theorem}

\begin{proof}
Truncate only the input atoms in a representation of $F_\star$ of mass at most $V$.
Exactly as in Lemma~\ref{lem:comparator}, the resulting comparator $F_N$ belongs to
$\Cclass^{\mathcal Y}$ and satisfies
$\norm{F_N-F_\star}_{L^2(\mu;\mathcal Y)}\le V\eta_N(\mu)$.
The fully-corrective optimization proof of Proposition~\ref{prop:fw} is Hilbertian and
uses only that the feasible set has empirical radius at most $V$, so it gives the same
$16V^2/(m+3)$ optimization gap for vector-valued squared loss.

For transfer to population risk use the centered loss from
Lemma~\ref{lem:vector-loss}. Its envelope is bounded by a universal multiple of
$(B+V)^2$. Symmetrization, Lemma~\ref{lem:vector-loss}, and bounded-difference
concentration yield a uniform empirical-population deviation of the last two terms in
\eqref{eq:vector-stat-uniform}. Finally,
\[
 L(F)-L(F_\star)=\norm{F-F_\star}_{L^2(\mu;\mathcal Y)}^2
\]
for Hilbert-valued conditional expectation, so comparing the empirical greedy iterate
with $F_N$ proves the result.
\end{proof}

Banach- and Hilbert-valued shallow neural approximation has important antecedents,
including Korolev's two-layer Banach-valued theory \cite{korolev2022}; operator-learning
error decompositions also appear in the DeepONet analysis of Lanthaler, Mishra, and
Karniadakis \cite{lanthaler2022}, with earlier operator universal approximation due to
Chen and Chen \cite{chenchen1995}. Theorem~\ref{thm:vector-stat} strengthens
the coordinatewise reduction by exploiting the rank-one structure of the vector atoms
directly; no output truncation is required.

\section*{Discussion and Conclusion}\label{sec:discussion}

\paragraph{What the result does and does not quantify.}
The fixed-radius theorem concerns the weighted regularity class $\Bsharp$. The span
preservation in \eqref{eq:same-span} is useful for density, but it does not make the
weighted ball equivalent to the unweighted one. The statement is therefore a
constructive approximation/learning result \emph{inside a smoothness class}, not a rate
for every $L^2$ target covered by an infinite-dimensional UAT. Proposition~\ref{prop:lipschitz-lb} makes part of this restriction explicit:
the weighted norm controls the Lipschitz seminorm in the embedding pseudometric. In the
one-dimensional threshold experiment, Corollary~\ref{cor:threshold-linear} pins the
weighted norm between $\lambda$ and $1+\lambda$, while the unweighted norm stays at
most one. Thus fixed weighted balls genuinely remove increasingly sharp thresholds in
that example.

\paragraph{Distribution-dependent resolution.}
The natural resolution quantity is $\eta_N(\mu)$, not the uniform tail
$\eta_N^\infty$. For highly concentrated input laws, $\eta_N(\mu)$ can be much smaller
than the envelope bound. This gives a direct route to data-distribution-dependent
resolution choices without altering the statistical-complexity argument.

\paragraph{Computational bottleneck.}
The main unresolved issue is the nonlinear neuron oracle. The statistical theorem is
uniform in $N$, but the parameter search is not. Bach's convex neural network analysis
already emphasizes that the unit-addition subproblem can be the computationally hard
part of an infinite-dimensional convex formulation \cite{bach2017}. The synthetic
experiment sidesteps this issue by replacing the continuous oracle with a finite pool.

\paragraph{Potential operator-learning application.}
Recent greedy operator-learning work focuses on linear operators and kernel estimation
\cite{lin2025}. The present scalar theory is formulated for nonlinear regression maps
of infinite-dimensional inputs. Combined with the vector-valued extension of
Section~\ref{sec:hilbert}, this suggests a route toward nonlinear operator or score
learning; Theorem~\ref{thm:vector-stat} already removes explicit output-dimension dependence from the statistical term. A tractable nonlinear neuron oracle and realistic operator-learning benchmarks remain necessary before making a strong application claim. For context, nonlinear operator
approximation and discretization error have a broader literature
\cite{chenchen1995,lanthaler2022,kovachki2021}.

\paragraph{Conclusion.}
A parameter-weighted variation class provides a clean setting in which three errors can
be separated for greedy neural learning from infinite-dimensional inputs: coordinate
truncation, finite greedy width, and finite sampling. The key statistical observation
is that an $\ell^2$ coordinate embedding has a Hilbert radius bounded independently of
the retained dimension, so a normalized dictionary admits a Rademacher estimate with
no explicit $N$ factor. This comes at two costs that should remain visible: the
weighted ball is more restrictive than the unweighted variation ball, and the nonlinear
neuron oracle is not computationally dimension-free. The quasi-Polish construction of
Galimberti supplies a natural motivating example with $\eta_N(\mu)^2\le N^{-1}$, while
the main quantitative proofs need only measurable coordinates with an $\ell^2$
envelope. The same argument extends to Hilbert-valued responses without output coordinate truncation, yielding a statistical term uniform in the output Hilbert dimension.


\begin{thebibliography}{99}

\bibitem{bach2017}
F.~Bach,
\newblock Breaking the curse of dimensionality with convex neural networks,
\newblock \emph{Journal of Machine Learning Research}, 18(19):1--53, 2017.

\bibitem{barron2008}
A.~R.~Barron, A.~Cohen, W.~Dahmen, and R.~A.~DeVore,
\newblock Approximation and learning by greedy algorithms,
\newblock \emph{Annals of Statistics}, 36(1):64--94, 2008.

\bibitem{bartlettmendelson2002}
P.~L.~Bartlett and S.~Mendelson,
\newblock Rademacher and Gaussian complexities: Risk bounds and structural results,
\newblock \emph{Journal of Machine Learning Research}, 3:463--482, 2002.

\bibitem{chenchen1995}
T.~Chen and H.~Chen,
\newblock Universal approximation to nonlinear operators by neural networks with arbitrary
activation functions and its application to dynamical systems,
\newblock \emph{IEEE Transactions on Neural Networks}, 6(4):911--917, 1995.

\bibitem{emawu2019}
W.~E, C.~Ma, and L.~Wu,
\newblock The Barron space and the flow-induced function spaces for neural network models,
\newblock arXiv:1906.08039, 2019.

\bibitem{galimberti2026}
L.~Galimberti,
\newblock $L^p$ approximation results for infinite dimensional Neural Networks,
\newblock arXiv:2608.08230, 2026.

\bibitem{jaggi2013}
M.~Jaggi,
\newblock Revisiting Frank--Wolfe: Projection-free sparse convex optimization,
\newblock in \emph{Proceedings of the 30th International Conference on Machine Learning},
2013.

\bibitem{klusowskibarron2016}
J.~M.~Klusowski and A.~R.~Barron,
\newblock Approximation by combinations of ReLU and squared ReLU ridge functions with
$\ell^1$ and $\ell^0$ controls,
\newblock arXiv:1607.07819, 2016.

\bibitem{korolev2022}
Y.~Korolev,
\newblock Two-layer neural networks with values in a Banach space,
\newblock \emph{SIAM Journal on Mathematical Analysis}, 54(6):6358--6389, 2022.

\bibitem{kovachki2021}
N.~Kovachki, S.~Lanthaler, and S.~Mishra,
\newblock On universal approximation and error bounds for Fourier Neural Operators,
\newblock arXiv:2107.07562, 2021.

\bibitem{kurkovasanguineti2001}
V.~K\r{u}rkov\'a and M.~Sanguineti,
\newblock Bounds on rates of variable-basis and neural-network approximation,
\newblock \emph{IEEE Transactions on Information Theory}, 47(6):2659--2665, 2001.

\bibitem{lanthaler2022}
S.~Lanthaler, S.~Mishra, and G.~E.~Karniadakis,
\newblock Error estimates for DeepONets: a deep learning framework in infinite dimensions,
\newblock \emph{Transactions of Mathematics and Its Applications}, 6(1):tnac001, 2022.

\bibitem{ledouxtalagrand1991}
M.~Ledoux and M.~Talagrand,
\newblock \emph{Probability in Banach Spaces: Isoperimetry and Processes},
\newblock Springer-Verlag, Berlin, 1991.

\bibitem{lin2025}
Y.~Lin, J.~Jia, Y.-J.~Lee, and R.~Zhang,
\newblock Orthogonal greedy algorithm for linear operator learning with shallow neural network,
\newblock arXiv:2501.02791, 2025.

\bibitem{maurer2016}
A.~Maurer,
\newblock A vector-contraction inequality for Rademacher complexities,
\newblock in \emph{Algorithmic Learning Theory (ALT 2016)}, Lecture Notes in Computer
Science, vol.~9925, Springer, 2016.

\bibitem{siegelxuvariation}
J.~W.~Siegel and J.~Xu,
\newblock Characterization of the variation spaces corresponding to shallow neural networks,
\newblock arXiv:2106.15002, 2021.

\bibitem{siegelxuoga}
J.~W.~Siegel and J.~Xu,
\newblock Optimal convergence rates for the orthogonal greedy algorithm,
\newblock arXiv:2106.15000, 2021.

\bibitem{vandervaartwellner1996}
A.~W.~van der Vaart and J.~A.~Wellner,
\newblock \emph{Weak Convergence and Empirical Processes: With Applications to
Statistics},
\newblock Springer-Verlag, New York, 1996.

\end{thebibliography}
\end{document}